\documentclass[article, 10pt, conference]{ieeeconf}   

\IEEEoverridecommandlockouts                              

\usepackage{import}
\usepackage[pdftex]{graphicx}
\usepackage{cite}
\usepackage{verbatim}		
\usepackage{amsmath}
\usepackage{amssymb}

\usepackage{amsthm}
\usepackage{mathtools}	
\usepackage{grffile}	
\usepackage[tight,footnotesize]{subfigure}
\usepackage{microtype} 
\usepackage{color}
\usepackage{url}
\usepackage[ruled,vlined,linesnumbered]{algorithm2e}
\usepackage{bm} 
\usepackage{comment}
\usepackage{arydshln}
\let\labelindent\relax
\usepackage{enumitem}

\usepackage{adjustbox}
\usepackage{tikz}
\usetikzlibrary{shadings, patterns, angles, quotes, arrows.meta, shapes, decorations.pathmorphing, decorations.shapes, decorations.text, automata, positioning}
\usetikzlibrary{calc,intersections,arrows.meta}
\usepackage{pgfplots}

\usepackage{hyperref}
\hypersetup{
	colorlinks=true,
	linkcolor=blue,
	citecolor=red,
}

	\tikzset{
	pil/.style={
		->,
		thick,
		shorten <=2pt,
		shorten >=2pt,}
}

\theoremstyle{remark}

\newtheorem{remarkx}{Remark}
\newenvironment{remark}
{\noindent
\pushQED{\qed}\remarkx}
{\popQED\endremarkx}

\newtheorem{examplex}{Example}
\newenvironment{example}
{\noindent\rule{8.6cm}{0.5pt}
\pushQED{\qed}\examplex}
{\popQED\endexamplex}

\theoremstyle{definition}
\newtheorem{defn}{Definition}

\newtheorem{problem}{Problem}
\newtheorem*{problem*}{Problem}

\theoremstyle{plain}
\newtheorem{theorem}{Theorem}
\newtheorem{lemma}{Lemma}
\newtheorem{coroll}{Corollary}
\newtheorem{prop}{Proposition}

\newcommand{\cmmnt}[1]{}

\newcommand{\scalemath}[2]{\scalebox{#1}{\mbox{\ensuremath{\displaystyle #2}}}}

\useRomanappendicesfalse

\title{\LARGE \bf SO(3) attitude controllers and the alignment of robots with non-constant 3D vector fields}

\author{Jesús Bautista, Hector Garcia de Marina} %

\begin{document}

\maketitle
\thispagestyle{empty}
\pagestyle{empty}



\section{Introduction} 
\label{sec: intro}

This technical note aims to introduce geometric controllers to roboticist for aligning \emph{3D robots} with non-constant 3D vector fields. This alignment entails the control of the robot's 3D attitude. We derive with \emph{excessive} detail all the calculations needed for the analysis and implementation of the controllers.

The sources for the technical results and formal proofs employed throughout this note can be found in \cite[Section 3.2]{modern_robotics} for a concise introduction to $\text{SO}(3)$, \cite{so3_catalanes} for Lie groups in robotics, and \cite{bullo_book} for a more profound understanding of geometric controllers.

\section{Preliminaries, robot kinematics and problem formulation}
\label{sec: pre}

We represent a robot in the 3D space with a \emph{Body-fixed} Cartesian frame of coordinates $\mathcal{F}_B$. The position and attitude of the robot with respect to the \emph{Earth-fixed} frame $\mathcal{F}_E$ are respectively described by the vector $p \in \mathbb{R}^3$ and the matrix

\begin{equation} \label{eq: R}
   R = \begin{bmatrix}x^E & y^E & z^E\end{bmatrix},
\end{equation}
where $x^E, y^E, z^E \in S^2$ are the orthonormal column vectors represented in $\mathcal{F}_E$ which generate the vector basis of $\mathcal{F}_B$ (see \autoref{fig: fe_vs_fb}). Indeed, the three more familiar Euler angles that describe the orientation of $\mathcal{F}_B$ with respect to $\mathcal{F}_E$ are encoded in the (rotation) matrix $R$ \cite[Section 4]{so3_tuto}. The main advantage of representing rotations with $R$, as opposed to Euler angles, lies in the fact that the attitude is uniquely defined globally; i.e., every physical orientation of our robot corresponds to a unique rotational matrix. 


In classical mechanics, we define the motion of an object as the variation of its body-fixed reference frame over time concerning an inertial reference frame, i.e., we need to determine/measure the linear and angular velocities of the robot relative to $\mathcal{F}_E$. We always have the freedom to represent inertial variables in more practical (although no convenient sometimes) frames of coordinates. For example, an Inertial Measurement Unit (IMU) measures linear accelerations and angular velocities relative to $\mathcal{F}_E$, but the reading is given in $\mathcal{F}_B$ since the IMU is onboard the robot.

\begin{figure}
\centering
\includegraphics[trim={3cm 3.8cm 3cm 6.5cm}, clip, width=0.8\columnwidth]{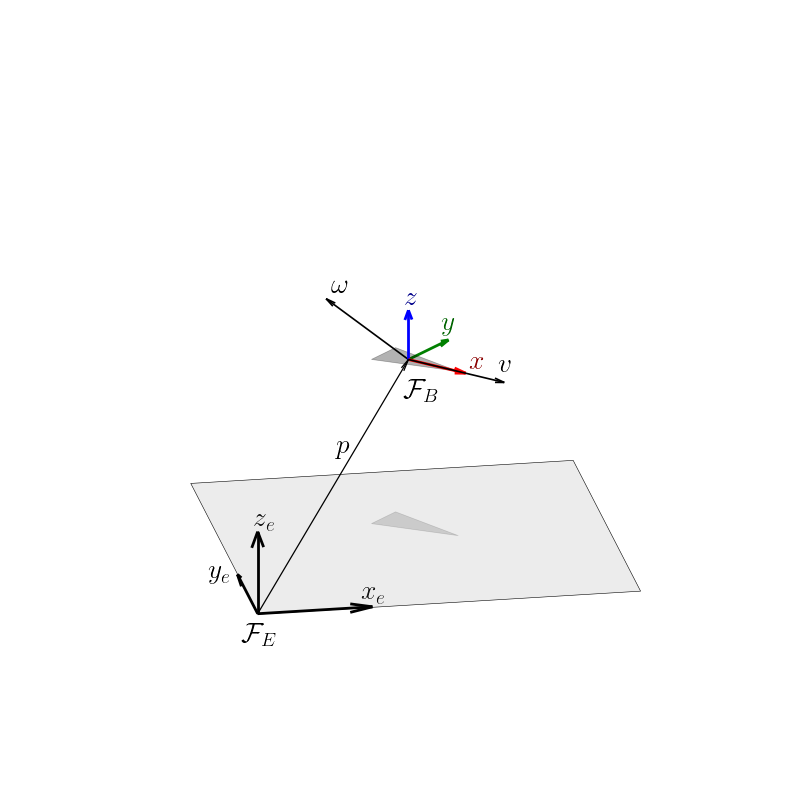}
\caption{This figure illustrates a 3D unicycle robot moving at a linear velocity $v$ and rotating at an angular velocity $\omega$. These velocities are observed from the inertial frame $\mathcal{F}_E$ but represented in the non-inertial frame $\mathcal{F}_B$, which is fixed at the body of the robot. $\{x_e, y_e, z_e\}$ and $\{x, y, z\}$ are the vector bases of $\mathcal{F}_E$ and $\mathcal{F}_B$, respectively.}
\label{fig: fe_vs_fb}
\end{figure}

In our scenario, considering $v^E \in \mathbb{R}^3$ and $\omega^E \in \mathbb{R}^3$ as the linear and angular velocities of the robot measured from $\mathcal{F}_E$ respectively, we denote its representation in $\mathcal{F}_B$ as  $v = \begin{bmatrix}v_{x} & v_{y} & v_{z}\end{bmatrix}^\top  \in \mathbb{R}^3$ and $\omega = \begin{bmatrix}\omega_{x} & \omega_{y} & \omega_{z}\end{bmatrix}^\top  \in \mathbb{R}^3$. Thus, considering the rotation matrix $R$ as the linear transformation that, when applied to a vector represented in $\mathcal{F}_B$, returns its representation in $\mathcal{F}_E$, we have that
\begin{equation} \label{eq: R_transform}
    v^E = R v, \quad \omega^E = R \omega,
\end{equation}
where we can easily extract the kinematics for the position $p$ given by
\begin{equation}
\dot p(t) = R(t)v.
\end{equation}

The evolution of the signal $R(t)$ depends on the actuation $\omega(t)$ and can be straightforwardly derived from the Rodrigues' formula for the rotation of an angle $\theta\in\mathbb{R}$ around a unit axis $l = \begin{bmatrix}l_x & l_y & l_z \end{bmatrix}^\top$, which can be written in terms of the matrix exponential as
\begin{equation} \label{eq: rod_exp}
R_l(\theta) = \exp(\theta S_l),
\end{equation}
where $S_l = \left[\begin{smallmatrix}0 & -l_z & l_y \\
l_z & 0 & -l_x \\
-l_y & l_x & 0
\end{smallmatrix}\right]$; therefore, we have that
\begin{equation}
\label{eq: Omega11}
\frac{\mathrm{d}}{\mathrm{dt}}R_l(\theta) = \dot\theta S_l \exp(\theta S_l) = \exp(\theta S_l) \dot\theta S_l = R_l(\theta)\Omega,
\end{equation}
where 
\begin{equation}
\label{eq: Omega1}
\Omega = \begin{bmatrix} 0 & -\omega_{z} & \omega_{y} \\ \omega_{z} & 0 & -\omega_{x} \\ -\omega_{y} & \omega_{x} & 0 \end{bmatrix}.
\end{equation}
Now, we are ready to write the kinematics of our 3D robot as
\begin{equation}
\label{eq: kin}
  \begin{cases}
       \begin{array}{ll}
           \dot p(t) = R(t) v \\
           \dot R(t) = R(t) \Omega \\
       \end{array},
   \end{cases}
\end{equation}
where $\Omega$ through $\omega$ is our actuation over the robot since we decided that $v$ will be constant.

\begin{remark} \label{rem: v}
In order to focus on the attitude control of the robot, we will consider $v$ in (\ref{eq: kin}) constant; in particular, $v = \begin{bmatrix}1 & 0 & 0\end{bmatrix}^\top.$
\end{remark}

Finally, we are ready to pose the following question.
\begin{problem} \label{prolem: main}
Given $R(t)$ as the signal for the 3D attitude representation of our robot and $R_a(p,t)$ as the \emph{attitude target}. What should be the actuation $\Omega(t)$ in (\ref{eq: kin}) through $\omega$ such that $R(t) \to R_a(p,t)$ as $t\to\infty$?
\end{problem}

In particular, we will construct $R_a(p,t)$ from a 3D vector field, and first, we will cover the trivial case where $v = 0$. The motivation for considering $R_a(p,t)$ comes from the necessity of some robotic problems where a mobile robot needs to track a given non-constant \emph{direction} given by a 3D vector field, e.g., in order to track a path \cite{yao2021singularity} or to track the source of a scalar field \cite{acuaviva2023resilient}.





\section{Brief introduction to Lie groups and $\text{SO}(3)$}

A group is a set $G$ of elements equipped with a binary operation ``$\cdot$'' that combines two elements $a,b\in G$ and outputs another element $c$ in the same set $G$, i.e., $a\cdot b = c$. This operation must be \emph{associative} and must count with an identity element $e$, i.e., $a\cdot e = a$, and for each element in $G$, there is an inverse so that the result of the binary operation is the identity, e.g., $a\cdot b = e$ where $b$, sometimes denoted $a^{-1}$, is the inverse of $a$.

A \textit{Lie group} joins together the concepts of \textit{group} and \textit{smooth manifold}, a differentiable topological space \cite{wikimani}. Sometimes, the elements of a Lie group can be associated with a physical meaning, e.g., the orientation of a reference frame, and that is why it makes them interesting in robotics. The local properties of a smooth manifold allow a coordinate system around any point (an element of the group). This fact enables us to identify a vector space around every element of the group, also known as \emph{tangent space}. For example, the Earth's surface is a manifold, and its flat maps are tangent spaces. The tangent space at the identity of the Lie group, endowed with a bilinear nonassociative product called the \textit{Lie bracket}, is known as the \textit{Lie algebra}. Indeed, a Lie group is almost completely determined by its Lie algebra. Thus, for many purposes, one may operate directly in the Lie algebra instead of with the elements of its Lie group.

In this note, we focus on the Lie group known as $\text{SO}(3)$. Each element of this group represents a rotation in 3D, and all possible rotations are present. In particular, each element $R$ of the group is a linear transformation $\mathbb{R}^3 \rightarrow \mathbb{R}^3$; indeed, the element $R$ is also represented and known as a \emph{rotation matrix}. It is a Lie group since we can associate each element $R$ of $\text{SO}(3)$ with a point in a smooth manifold. Although it is beyond the scope of this note to show it, for the sake of completeness, we note that the manifold $\text{SO}(3) \simeq \mathrm{RP}^3$ and its fundamental group is $\pi\left(\text{SO}(3)\right) = \mathbb{Z}/2\mathbb{Z}$. 

\subsection{The special Orthogonal Group $\text{SO}(3)$} \label{sec: so3}

Let us define the group $\text{SO}(3)$ as
\begin{equation}
\label{eq: SO3}
    \text{SO}(3) := \{R \in \mathbb{R}^{3 \times 3} \,|\, R^\top R = I,\, \det(R) = 1\}.
\end{equation}
Since $\text{SO}(3)$ is also a smooth manifold, i.e., it is a Lie group, there exists a tangent space at every point $R$, and the velocity of $R$, i.e., $\dot R$, belongs to such a space. We can calculate $\dot R$ by taking the time derivative of the first constraint in (\ref{eq: SO3}) as
\begin{align}
\dot R^\top R + R^\top\dot R &= 0 \nonumber \\
\Leftrightarrow \dot R &= - R \dot R^\top R \nonumber \\
\Leftrightarrow \dot R &= R \Omega, \label{eq: RdotOmega}
\end{align}
where $\Omega$ has already been found in (\ref{eq: Omega11}), and we note that it is indeed skew-symmetric since $\dot R ^\top R = -(\dot R ^\top R)^\top$. Looking at the definition, the Lie algebra is the tangent space associated with $\text{SO}(3)$ at the identity element of the group, i.e., setting $R=I$ in (\ref{eq: RdotOmega}) we have that $\Omega$ is in the Lie algebra of $\text{SO}(3)$. We formally denote the tangent spaces of $\text{SO}(3)$ as
\begin{equation}
    \mathrm{T}_R\text{SO}(3) := \{R\Omega \, | \, \Omega \in \mathfrak{so}(3)\},
\end{equation}
where $\mathfrak{so}(3)$ is the vector space of the Lie algebra. This vector space has dimension $3$ because each of its elements can be written as
\begin{equation}
\Omega = \omega_x \left[\begin{smallmatrix}0 & 0 & 0 \\
0 & 0 & -1 \\
0 & 1 & 0
\end{smallmatrix}\right] + \omega_y \left[\begin{smallmatrix}0 & 0 & 1 \\
0 & 0 & 0 \\
-1 & 0 & 0
\end{smallmatrix}\right] + \omega_z \left[\begin{smallmatrix}0 & -1 & 0 \\
1 & 0 & 0 \\
0 & 0 & 0
\end{smallmatrix}\right],
\end{equation}
where the coordinates can be identified as the three angular velocities in $\omega$. Note that there is a bijective map between $\mathfrak{so}(3)$ and $\mathbb{R}^3$, i.e., $\mathfrak{so}(3) \cong \mathbb{R}^3$. This fact allows us to (re)define this vector space as
\begin{equation}
    \mathfrak{so}(3) := \{x^\wedge \in \mathbb{R}^{3 \times 3} \, | \, x \in \mathbb{R}^{3},\, x^\wedge = - (x^\wedge)^\top \},
\end{equation}
where we have introduced the linear map \emph{hat} $\wedge : \mathbb{R}^3 \rightarrow \mathfrak{so}(3)$ as the transformation of a given vector $x \in \mathbb{R}^3$ to a skew-symmetric matrix $x^\wedge \in \mathbb{R}^{3\times 3}$ similarly as we have constructed $\Omega$ in (\ref{eq: Omega1}) from $\omega$. This linear map is interesting since we have the following identity
\begin{equation} \label{eq: cross_identity}
    x^\wedge y = x \times y,
\end{equation}
for any $y \in \mathbb{R}^3$, where $\times$ denotes the standard cross product.

To finish the bijection between $\mathfrak{so}(3)$ and $\mathbb{R}^3$, for the \emph{recovery} of $\omega$ from a given $\Omega$ we define the linear map \emph{vee} $\vee : \mathfrak{so}(3) \rightarrow \mathbb{R}^3$ as the opposite of the hat map, i.e., $\omega = \Omega^\vee$.





\subsection{Geodesics}

Given two elements $q_a,q_b \in M$, where $M$ is a smooth manifold and $t_a,t_b \in \mathbb{R^+}$ with $t_a < t_b$, we let

\begin{align}
    \mathcal{C}(q_a, q_b) := \{ \gamma(t): [t_a,t_b] \rightarrow M \; | \; & \gamma(t_a) = q_a, \gamma(t_b) = q_b, \nonumber \\ &\gamma \in C^1\},
\end{align}
be the set of $C^1$-curves $\gamma$ in $M$ which connects $q_a$ with $q_b$. We define the \textit{Riemannian length} of a curve $\gamma$ as
\begin{equation} \label{eq: rieman_len}
    \ell (\gamma) = \int_{t_a}^{t_b} \sqrt{\langle \; \dot\gamma(t), \dot\gamma(t)\; \rangle_{\gamma(t)}} \, \mathrm{d}t,
\end{equation}
where $\langle \cdot , \cdot \rangle_x : T_x M \times T_x M \rightarrow \mathbb{R}$ is an inner product on $T_x M$ of $M$. A \textit{Riemannian metric} $\langle \cdot , \cdot \rangle$ of $M$ is a smoothly chosen inner product $\langle \cdot , \cdot \rangle_x$ on each of the tangent spaces $T_x M$ of $M$. 

Among all the curves $\gamma \in \mathcal{C}(q_a, q_b)$, we are interested in those with the minimum length since they define a \textit{Rimannian distance} between $q_a,q_b \in M$ as
\begin{equation} \label{eq: rieman_dis}
    d_M(q_a,q_b) = \mathrm{inf}\{\ell(\gamma) \; | \; \gamma \in \mathcal{C}(q_a, q_b), \; q_a, q_b \in M\}.  
\end{equation}
In the following example, we show that this general definition of distance matches with the particular case of Euclidean distance.

\begin{example} \label{ex: line}
    Consider $M = \{(x,y) \in \mathbb{R}^2\}$, and two points $q_a, q_b \in \mathbb{R}^2$ connected by a curve $\gamma(t) = (x(t), y(t)) \subset \mathbb{R}^2$ such that $\gamma(0) = q_a = (x_a, y_a)$ and $\gamma(\Delta t) = q_b = (x_b, y_b)$. Since we are in $\mathbb{R}^2$, let $\langle \cdot , \cdot \rangle$ be the well-known inner product between two vectors. In this scenario, it is straightforward to identify that the curve $\gamma(t)$ with minimum length is the one with constant velocity $\dot \gamma(t)$ for every $t \in [0,\Delta t]$, i.e., a straight line. Moreover, the length of such a curve as defined in \eqref{eq: rieman_len} yields
    \begin{align*}
        \int_{0}^{\Delta t} \sqrt{\left(\dot x(t)\right)^2 + \left(\dot y(t)\right)^2} \mathrm{d}t &= \sqrt{\left(\dot x \Delta t\right)^2 + \left(\dot y \Delta t\right)^2}\\
        &= \sqrt{\left(x_b - x_a\right)^2 + \left(y_b - y_a\right)^2},
    \end{align*}
    which is the Pythagorean equation. Therefore, we have that the length of our chosen $\gamma$ is in fact the Euclidean distance between $\gamma(0)$ and $\gamma(\Delta t)$.
    
\end{example}

The straight line within a flat manifold of Example \ref{ex: line} is the most illustrative case of \textit{geodesic}. Formally speaking, the geodesics are curves that allow us to go from $q_a$ to $q_b$ while maintaining a constant velocity in $M$; i.e., they are \emph{zero-acceleration} curves \cite[Definition 3.102]{bullo_book}; of course, the acceleration here refers to within the manifold strictly. Indeed, it is possible to show for every manifold $M$ that the curves with minimum length are always geodesics \cite{do1992riemannian}, so we define the length-minimizing geodesic as follows.

\begin{defn}[Length-minimizing geodesic]
    Let $M$ be a manifold with a Riemannian metric $\langle \cdot , \cdot \rangle$. The length-minimizing geodesic, denoted as $\gamma_g$, is a $C^1$-curve $\gamma : [t_a,t_b] \rightarrow M$ with the property that $\ell(\gamma_g) = d_M(\gamma_g(t_a),\gamma_g(t_b))$.
\end{defn}

\begin{figure}[t]
    \centering
    \includegraphics[trim={2.3cm 3.5cm 1.8cm 3.5cm}, clip, width=\columnwidth]{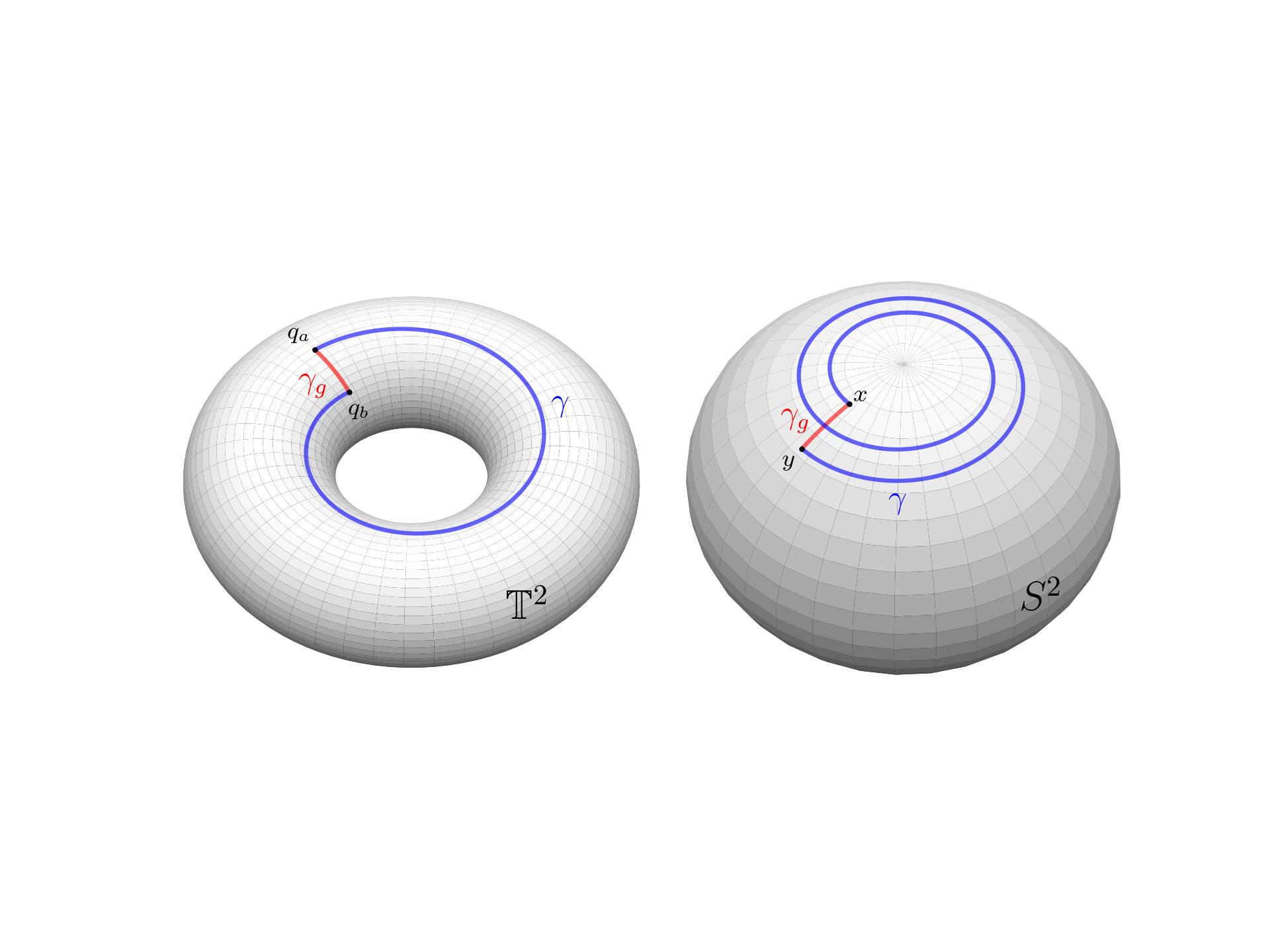}
    \caption{This figure illustrates, for two different manifolds, a pair of geodesics ($\gamma$ in blue color and $\gamma_g$ in red color)  that connects $q_a, q_b \in \mathbb{T}^2$ (left) and $x,y \in S^2$ (right). Note that $\gamma_g$ is the length-minimizing geodesic.}
    \label{fig: geodesics}
\end{figure}

Let us note that the definition of length in \eqref{eq: rieman_len} and the geodesics has a potential physical meaning. For any system whose state $x\in M$, we can design $\langle  \dot\gamma(t), \dot\gamma(t) \rangle$ according to their dynamics to represent kinetic energy along $\gamma$. Hence, if $\gamma$ is a length-minimizing geodesic, then $\ell (\gamma)$ minimizes the kinetic energy integral. Moreover, as the geodesics are zero acceleration curves, they represent the trajectories of free particles within the manifold, where motion is entirely determined by the curvature of $M$. This concept aligns with the fundamental idea in general relativity, where particles move along geodesics within space-time, and the curvature of such a manifold is induced by the presence of a mass (gravity).

Even when $M \subset \mathbb{R}^3$, e.g., $\mathbb{T}^2$ in \autoref{fig: geodesics}, visualizing the geodesics is not straightforward. For an illustrative example, consider the unit sphere $S^2 := \{x \in \mathbb{R}^3 \, | \, \|x\| = 1\}$ and two points $x,y \in S^2$, as in \autoref{fig: geodesics}. The geodesic $\gamma_g \in \mathcal{C}(x, y)$ can be interpreted geometrically as the (not necessarily unique) shortest arc in the unit sphere that connects $x$ with $y$, whose length, the geodesic distance in $S^2$ as defined in \eqref{eq: rieman_dis}, can be derived (see Appendix \ref{ap: s2}) as 
\begin{equation} \label{eq: dis_s2}
    d_{S^2}(x, y) = \arccos(x^\top y),
\end{equation}
where $\arccos(\cdot)$ takes values between $0$ and $\pi$.

 In $\text{SO}(3)$ the geodesic is not that straightforward to identify, but we can follow the same steps as in $S^2$ to find a geodesic distance. Let $\langle \Omega_1, \Omega_2 \rangle = \mathrm{tr}(\Omega_1^\top \Omega_2) =: \langle\langle \Omega_1, \Omega_2 \rangle\rangle$ be our Riemannian metric for $\text{SO}(3)$, being $\Omega_1, \Omega_2 \in \mathfrak{so}(3)$. Consider a constant angular velocity $\omega$ that \textit{connects} $R_a,R_b \in \text{SO}(3)$ along $\gamma_g \in \mathcal{C}(R_a, R_b)$, i.e., $\dot \gamma_g(t) = \omega$, such that $\gamma(0) = R_a$ and $\gamma(\Delta t) = R_b$. Then, computing the length of such a curve as in \eqref{eq: rieman_len} yields
\begin{align} \label{eq: len_so3}
    \int_{0}^{\Delta t} \sqrt{\langle \dot \gamma(t), \dot \gamma(t) \rangle_{\gamma(t)}} \, \mathrm{d}t \ &= \Delta t \sqrt{\mathrm{tr}({\omega^\wedge}^\top {\omega^\wedge})} \nonumber \\
    &= \Delta t \|\Omega\|_F = \|\tau^\wedge\|_F,    
\end{align}
where $\|.\|_F$ is the \textit{Frobenius norm} and $\tau^\wedge = \Omega \Delta t$ is the element of $T_{R_a}\text{SO}(3)$ that connects $R_a$ with $R_b$, as we will show in the following subsection. Indeed, let us note the following identity: since $\Omega = \omega^\wedge$ is a skew-symmetric matrix, then
\begin{equation} \label{eq: frob_product}
    \|\omega^\wedge\|_F = \sqrt{\mathrm{tr}({\omega^\wedge}^\top \omega^\wedge)} = \sqrt{2(\omega \cdot \omega)},
\end{equation}
where $\cdot$ denotes the standard scalar product between two vectors in the Euclidean space.

The length computed in \eqref{eq: len_so3} is the geodesic distance in $\text{SO}(3)$, expressed in terms of the integrated velocity vector $\tau$. Ideally, we would prefer to express this distance in terms of $R_a$ and $R_b$ directly, as we did in Example \ref{ex: line} for a flat manifold, but to do so we have to introduce first the maps that enable us to establish such a relationship between $\mathfrak{so}(3)$ and $\text{SO}(3)$.

\subsection{Exponential and logarithmic maps} \label{sec: exp_log}

The calculation of the velocity of $R$, as formulated in \eqref{eq: RdotOmega}, leads to an ordinary differential equation (ODE) whose solution is the signal
\begin{equation} \label{eq: exp_0}
    R(t) = R(0) \exp(\omega^\wedge t) = R(0) \exp(\tau^\wedge (t)),
\end{equation}
where we denote $\tau(t) = \omega t$ as the \textit{integrated rotation vector} or \textit{tangent vector}. Since $R(t), R(0) \in \text{SO}(3)$, the associative property of the group yields 
$$\exp(\tau^\wedge (t)) = R(0)^T R(t) \in \text{SO}(3).$$
Therefore, $\exp(\tau^\wedge (t))$ maps elements $\tau^\wedge (t)$ of the Lie algebra to the group. This is known as the \textit{exponential map}. Intuitively, as \autoref{fig: exp_log} illustrates for $S^1$, this map wraps the tangent vectors over the manifold following the length-minimizing geodesic from $x$ to $y$.

\begin{figure}[t]
    \centering
    \includegraphics[trim={2.5cm 1.4cm 1cm 2.5cm}, clip, width=\columnwidth]{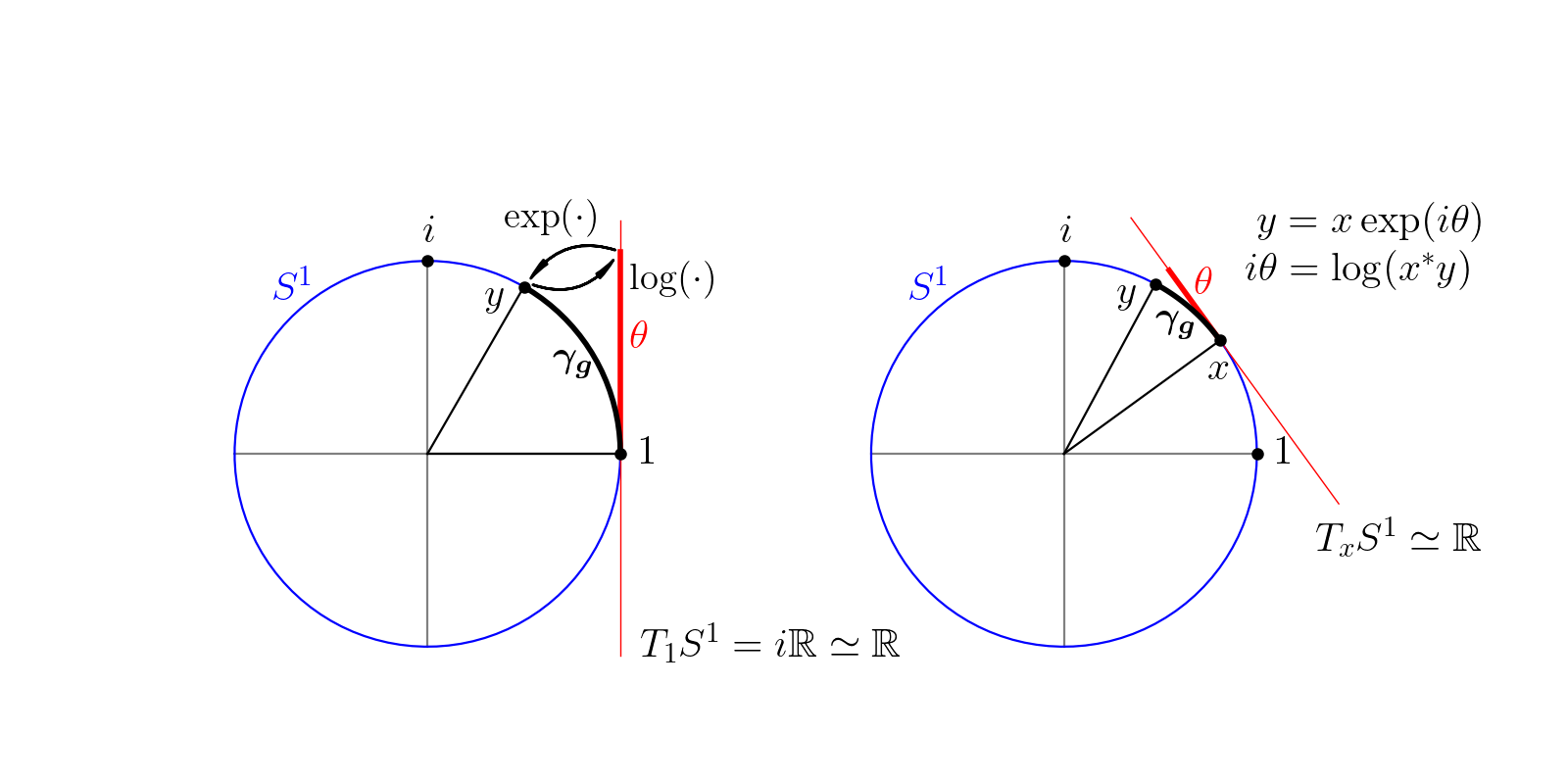}
    \caption{The $S^1$ manifold is a unit circle (blue) in $\mathbb{C}$, which contains every complex number $x$ such that $x^*x=1$. Its Lie algebra $T_1S^1$ is the line of imaginary numbers $i\mathbb{R}$ (red), and any tangent space $T_xS^1$ is isomorphic to the line $\mathbb{R}$ (red). Tangent vectors (red segments) wrap the manifold creating $\gamma_g$. The $\exp(\cdot)$ and $\log(\cdot)$ maps (arrows) wrap and unwrap elements of $\mathbb{R}$ to/from elements of $S^1$. This figure has been adapted from the reference \cite{so3_catalanes}.}
    \label{fig: exp_log}
\end{figure}

By taking the absolutely convergent Taylor series
\begin{equation} \label{eq: exp_taylor}
    \exp(\tau^\wedge) = \sum_{k=0}^\infty \frac{(\tau^\wedge)^k}{k!}
\end{equation}
and computing it as in \cite{so3_catalanes}, we formulate the $\exp(\cdot)$ map as
\begin{equation} \label{eq: exp}
     \exp(\tau^\wedge) = I + \frac{\sin\theta}{\theta} \tau^\wedge +\frac{1 - \cos\theta}{\theta^2} ({\tau^\wedge})^2,
\end{equation}
where $\theta = \|\tau\|$ is denoted as the \textit{integrated rotation angle}. This closed form of the exponential map is the well-known Rodrigues' formula for rotation matrices, as previously introduced in \eqref{eq: rod_exp}.

The inverse of $\exp(\cdot)$ is what we know as the \textit{logarithmic map} $\log(\cdot)$ and its closed form can be obtained as follows. For the sake of clarity, let $R(0) = I$, i.e., $\tau^\wedge \in \mathfrak{so}(3)$, so we directly have that $\exp(\tau^\wedge(t)) = R(t)$. Therefore, we can consider the skew-symmetric/symmetric decomposition of $R$ to regroup the terms in \eqref{eq: exp} as
\begin{equation}\label{eq: skew_symm}
    R = \underbrace{\frac{\sin\theta}{\theta} \tau^\wedge}_{\frac{R - R^T}{2}} + \underbrace{ I +  \frac{1 - \cos\theta}{\theta^2} ({\tau^\wedge})^2}_{\frac{R + R^T}{2}}.
\end{equation}
Thus, we have that $\frac{\sin\theta}{\theta} \tau^\wedge = \frac{R - R^T}{2}$, which directly yields
\begin{equation} \label{eq: log}
     \tau^\wedge = \log(R) =
     \begin{cases}
        \begin{array}{ll}
            0_{3 \times 3} & R = I\\
            \left( \frac{\theta}{2 \sin\theta}  \right) (R - R^\top ) & R \neq I\\
        \end{array}.
    \end{cases}
\end{equation}
Here, $\tau$ is unknown and we cannot compute $\theta$ from it directly, but note that taking the trace of $R$ in \eqref{eq: skew_symm} leads us to $\cos\theta = \frac{1}{2}(\mathrm{tr}(R) - 1)$, which restricting $\theta \in [0,\pi]$ gives the unique solution
\begin{equation} \label{eq: arccos}
    \theta = \arccos\left(\frac{\mathrm{tr}(R) - 1}{2}\right).
\end{equation}
At the boundary $\theta = \pi$, $R$ is purely symmetric, i.e., $R = R^T$, so the skew matrix provides no information, and thus, the solution of $\log(\cdot)$ is not unique. Therefore, we have a unique $\log(\cdot)$ for all $\tau$ within the interior ball of radius $\pi$, i.e., when $\|\tau\| < \pi$. Taking this into consideration, we formally define the exponential and logarithmic linear maps as
\begin{align*}
    \exp :& \; \Pi_{\mathfrak{so}(3)} \rightarrow \Pi_{\text{SO}(3)} \nonumber \\
    \log :& \; \Pi_{\text{SO}(3)} \rightarrow \Pi_{\mathfrak{so}(3)},
\end{align*}
whose closed formed are given by \eqref{eq: exp} and \eqref{eq: log}, respectively. Here, $\Pi_{\mathfrak{so}(3)} := \{\tau^\wedge \in \mathfrak{so}(3) \, : \, \tau \in \mathbb{R}^{3},\, \|\tau\| < \pi\}$ and $\Pi_{\text{SO}(3)} := \{R \in \text{SO}(3) \, : \, \mathrm{tr}(R) \neq -1\}$. 

When $\tau^\wedge \in T_{R_a} \text{SO}(3)$, we have that $R(0) = R_a$ in \eqref{eq: exp_0}, i.e., $\exp(\tau^\wedge(t)) = R_a^\top R(t)$. Computing the inverse of this expression, as we did for the trivial case of $R(0)=I$, leads us to $\tau^\wedge(t) = \log(R_a^\top R(t))$. This same scenario is illustrated for $S^1$ in \autoref{fig: exp_log} when $i\theta \in T_xS^1$. Therefore, revisiting \eqref{eq: len_so3}, we can now formalize the geodesic distance in $\text{SO}(3)$ in terms of $R_a$ and $R_b$ as
\begin{equation} \label{eq: metric}
    d_{\text{SO}(3)}(R_a, R_b) = \|\log(R_a^\top R_b)\|_F,
\end{equation}
where $R_a^\top R_b \in \Pi_{\text{SO}(3)}$, i.e., $\mathrm{tr}(R_a^\top R_b) \neq -1$. Note that ensuring $\mathrm{tr}(R_a^\top R_b) \neq -1$ is not straightforward, but there are particular situations where it becomes easier, e.g., see the following technical result.

\begin{lemma} \label{lemm: tr_log}
    Given $R_a(t), R_b(t) \subset \text{SO}(3)$ such that
    \begin{equation} \label{eq: dso3_exp}
        d_{\text{SO}(3)}(R_a(t), R_b(t)) \leq d_{\text{SO}(3)}(R_a(0), R_b(0)),
    \end{equation}
    where $t\in \mathbb{R}^+$. If $\mathrm{tr}(R_a(0)^\top R_b(0)) \neq -1$, then $\mathrm{tr}(R_a(t)^\top R_b(t)) \neq -1$ for all $t$.
\end{lemma}
\begin{proof}
    From \eqref{eq: dso3_exp} we have that $d_{\text{SO}(3)}(R_a(t), R_b(t))$ is bounded by the value at $t=0$.  On the other hand, from \eqref{eq: arccos} we have that $R$ is the furthest element in $\text{SO}(3)$ from $I$ when $\mathrm{tr}(R) = -1$. Therefore, the distance between $R_a(t)$ and $R_b(t)$ will not ever be the maximum one; hence, the claim.
\end{proof}

\subsection{Computing the exponential and logarithmic maps}

In practice, numerous methods exist for computing these two maps, and its choice may depend on practical requirement details. Since in $\text{SO}(3)$ we have the Rodrigues' formula, the implementation of the exponential map reduces to computing $\theta = \|\tau\|$ and applying \eqref{eq: exp}. Nonetheless, when $\theta \rightarrow 0$, the denominators of the coefficients in \eqref{eq: exp} also approach zero, so the computation of $\exp(\cdot)$ may be compromised numerically. In such cases, since we are dealing with small angles, it is convenient to truncate the series \eqref{eq: exp_taylor} up to a suitable order.

When we are working with the logarithmic map, we also have to take special care about these singularities. Implementing $\log(\cdot)$ also reduces to compute $\theta$ as in \eqref{eq: arccos} and apply \eqref{eq: log}, but when $\theta \rightarrow 0$ we have the same numerical problem. In that case, it is convenient to use the Taylor's series expansion 
\begin{equation*}
\left( \frac{\theta}{\sin\theta}  \right) = 1 + \frac{1}{6}\theta^2 + \frac{7}{360}\theta^4 + \mathcal{O}(\theta^6).
\end{equation*}
On the other hand, we have also seen that $\log(\cdot)$ is not well defined when $\theta = \pi$, but we can computationally deal with that as follows. When $\theta = \pi$, the symmetric part of \eqref{eq: skew_symm} yields $R = I + 2(\tau^\wedge)^2$, which can be rewritten as
\begin{equation}
    \tau^\wedge {\tau^\wedge}^\top = \frac{1}{2}(R + I)
\end{equation}
by considering that $(\tau^\wedge)^2 = \tau^\wedge {\tau^\wedge}^\top - I$. 
Thus, we can still determine $\tau$ up to a factor of $\pm 1$.

The reader can explore an implementation example in Python for both the exponential and logarithmic maps in the GitHub repository \cite{repo_ss}. The method we have chosen in this project to compute the logarithmic map is similar to the approach presented in \cite{log_roll}.

\subsection{The adjoint map}


The relationship between $\Omega \in \mathfrak{so}(3)$ and its representation in $\mathcal{F}_E$ as in \eqref{eq: R_transform}, denoted as $\Omega^E$, can be seen as operating with a different basis. Indeed, knowing that $R^{-1} = R^T$, we have that
\begin{equation} \label{eq: r_omega_rt}
    \Omega^E = R \Omega R^{T},
\end{equation}
that is, we first change the basis from $\mathcal{F}_E$ to $\mathcal{F}_B$ with $R^T$, then we apply $\Omega$, and finally we \emph{come back} to $\mathcal{F}_E$ with $R$. Of course, we note that $\Omega^E \in \mathfrak{so}(3)$ as it comes from a linear transformation. In the context of Lie groups, this linear transformation is the well-known \emph{adjoint map}, defined as the linear map $\mathrm{Ad}_R : \mathfrak{so}(3) \rightarrow \mathfrak{so}(3)$ such that $\mathrm{Ad}_R(\Omega) = R \Omega R^T$. Indeed, taking the vee map of \eqref{eq: r_omega_rt} leads to \eqref{eq: R_transform} by definition, i.e.,
\begin{equation} \label{eq: adjoint_vee}
    {\mathrm{Ad}_R(\Omega)}^\vee = R \Omega^\vee = R\omega,
\end{equation}
since $(\Omega^E)^\vee = \omega^E$. This is a powerful identity of the adjoint map that we will extensively use later in our analysis.

Taking the derivative of the adjoint map at $R=I$ gives the \textit{adjoint operator} $\mathrm{ad}_\Omega : \mathfrak{so}(3) \rightarrow \mathfrak{so}(3)$. We show next in Example \ref{ex: lie_bracket} that the definition of this operator can be formulated as
\begin{equation} \label{eq: ad_op_def}
    \mathrm{ad}_\Omega = \frac{\mathrm{d}}{\mathrm{d}t} \mathrm{Ad}_{e^{t\Omega}}\bigg|_{t=0},
\end{equation}
where $e^{t\Omega} := \exp(t\Omega)$ is a length-minimizing geodesic curve, and that
\begin{equation} \label{eq: ad_op}
    \mathrm{ad}_{\Omega_a} (\Omega_b) = [\Omega_a, \Omega_b] = \Omega_a \Omega_b - \Omega_b \Omega_a,
\end{equation}
where $[.\,,\,.] : \mathfrak{so}(3) \times \mathfrak{so}(3) \rightarrow \mathfrak{so}(3)$ is the Lie bracket operator of the Lie algebra. Note that this expression directly leads us to $\mathrm{ad}_{\Omega_a}(\Omega_b) = 0$ when $\Omega_a = \Omega_b$. Moreover, if the Lie algebra is $\mathfrak{so}(3)$, then \eqref{eq: cross_identity} yields $[\omega^\wedge_a, \omega^\wedge_b] = (\omega_a \times \omega_b)^\wedge$. Finally, considering the Taylor expansion of the exponential, note that the adjoint map can be written as
\begin{align} \label{eq: Ad_series}
    \mathrm{Ad}_{\exp(\Omega_a)}(\Omega_b) &= \sum^\infty_{m=0}\frac{(\mathrm{ad}_{\Omega_a})^m}{m!}(\Omega_b) \nonumber\\
    &= \Omega_b + [\Omega_a,\Omega_b] + \frac{1}{2}[\Omega_a,[\Omega_a,\Omega_b]] + \text{H.O.T.}
\end{align}

\begin{example} \label{ex: lie_bracket}
Consider two signals $R_a(t), R_b(t) \in \text{SO}(3)$ such that $R_a(0) = R_b(0) = I$, $\dot R_a(0) = \Omega_a$ and $\dot R_b(0) = \Omega_b \in \mathfrak{so}(3)$; e.g., the length-minimizing geodesics $R_a(t) = \exp(t\Omega_a)$ and $R_b(t) = \exp(t\Omega_b)$. Let us choose an arbitrary $s\in \mathbb{R}^+$ and define the curve 
$C_s(t) = \mathrm{Ad}_{R_a(s)}(R_b(t)) = R_a(s) R_b(t) R_a(s)^\top.$ 
When $t=0$, note that $C_s(t)$ is contained in $\mathfrak{so}(3)$, and thus its derivative with respect to $t$ or $s$. Hence, the derivative at $t = s = 0$ yields
\begin{align} \label{eq: ex_lie_bracket}
    \scalemath{0.9}{
    \frac{\mathrm{d}}{\mathrm{d}s}}&
    \scalemath{0.9}{
    \left( R_a(s) \Omega_b R_a^\top(s) \right)\bigg|_{s=0} = \frac{\mathrm{d}}{\mathrm{d}s}\left( \exp({s\Omega_a}) \Omega_b {\exp({s\Omega_a})}^\top \right)\bigg|_{s=0}} \nonumber\\ 
    &= \dot R_a(0) \Omega_b R_a^\top(0) + R_a(0) \Omega_b \dot R_a^\top(0) \nonumber \\
    &= \Omega_a \Omega_b - \Omega_b \Omega_a \nonumber \\
    &= [\Omega_a, \Omega_b] \in \mathfrak{so}(3),
\end{align}
thus $[\Omega_a, \Omega_b]$ is an interesting operation of $\mathfrak{so}(3)$. Indeed, from \eqref{eq: ex_lie_bracket} one can directly derive \eqref{eq: ad_op_def} and \eqref{eq: ad_op}.

\end{example}

\subsection{Adjoint invariant metric}
We say that a Riemannian metric $\langle . , .\rangle$ is \textit{Ad-invariant} when
\begin{equation} \label{eq: ad-inv}
    \langle \mathrm{Ad}_R(\Omega_1) , \mathrm{Ad}_R(\Omega_2) \rangle = \langle \Omega_1 , \Omega_2 \rangle.
\end{equation}
For example, considering the metric $\langle\langle \Omega_1, \Omega_2 \rangle\rangle = \mathrm{tr}(\Omega_1^\top \Omega_2)$ we have that
\begin{align*}
    \langle\langle \mathrm{Ad}_R(\Omega_1) , \mathrm{Ad}_R(\Omega_2) \rangle\rangle &= \mathrm{tr}(R\Omega_1^TR^T R \Omega_2 R^T) \\
    &= \mathrm{tr}(\Omega_1^T \Omega_2) = \langle\langle \Omega_1 , \Omega_2 \rangle\rangle.
\end{align*}
The metrics with this property will be particularly interesting in our analysis because they lead to the following identity.

\begin{lemma} \label{lem: ad_inv}
    \cite[Chapter 17]{CIS610_notes} Considering a Riemannian metric $\langle . , .\rangle$ of a smooth manifold $M$, if it is Ad-invariant then
    \begin{equation}
    \langle \mathrm{ad}_X(Y) , Z \rangle = - \langle Y , \mathrm{ad}_X(Z) \rangle,
    \end{equation}
    where $X,Y,Z \in T_x M$, $x\in M$.
\end{lemma}
\begin{proof}
    Given the definition of adjoint operator as in \eqref{eq: ad_op_def}, we have that
    \begin{align*}
     \langle \mathrm{ad}_{X}(Y) , Z \rangle &= \langle \frac{\mathrm{d}}{\mathrm{d}t} \mathrm{Ad}_{e^{tX}}(Y)\bigg|_{t=0} , Z \rangle  \\ 
     &= \frac{\mathrm{d}}{\mathrm{d}t}\langle  \mathrm{Ad}_{e^{tX}}(Y) , Z \rangle \bigg|_{t=0}.
    \end{align*}
    Here, since $\langle . , . \rangle$ is Ad-invariant, we can apply \eqref{eq: ad-inv}, which yields
    \begin{align*}
    \langle  \mathrm{Ad}_{e^{tX}}(Y) , Z \rangle &= \langle \mathrm{Ad}_{e^{-tX}}(\mathrm{Ad}_{e^{tX}}(Y)) , \mathrm{Ad}_{e^{-tX}}(Z) \rangle\\
     &= \langle  Y , \mathrm{Ad}_{e^{-tX}}(Z) \rangle.
    \end{align*}
    Therefore, we can directly conclude that
    \begin{align*}
     \langle \mathrm{ad}_{X}(Y) , Z \rangle &= \frac{\mathrm{d}}{\mathrm{d}t}\langle  \mathrm{Ad}_{e^{tX}}(Y) , Z \rangle \bigg|_{t=0} \\ 
     &= \frac{\mathrm{d}}{\mathrm{d}t}\langle  Y , \mathrm{Ad}_{e^{-tX}}(Z) \rangle \bigg|_{t=0} \\
     &= \langle Y , -\mathrm{ad}_{X}(Z) \rangle = - \langle Y , \mathrm{ad}_{X}(Z) \rangle.
    \end{align*}
\end{proof}

\subsection{The Jacobian of the exponential map in $\text{SO}(3)$}


Given $R(t)$ as a smooth curve of $\text{SO}(3)$, we now want to discover the relationship between the time derivative of $\tau^\wedge(t) = \log(R(t))$ and the velocity tensor of $R(t)$ from $\mathcal{F}_B$, denoted as $\Omega = \dot R(t) R(t)^\top$. One may show as in \cite[Theorem 1]{bullo_murray_1995} that
\begin{equation} \label{eq: omega_tau_int}
    \Omega = \int_0^1 \mathrm{Ad}_{e^{-\lambda \tau^\wedge(t)}}(\dot\tau^\wedge) \, \mathrm{d}\lambda,
\end{equation}
which, looking at the identity \eqref{eq: adjoint_vee}, can be interpreted as the integration of $\dot\tau$ along the length-minimizing geodesic that connects $I$ with $\exp(-\tau^\wedge(t))$. Following the proof of \cite[Theorem 2]{bullo_murray_1995} we can compute the solution of the integral in \eqref{eq: omega_tau_int}, and it yields
\begin{equation} \label{eq: series_derv}
    \dot\tau^\wedge = \sum^\infty_{n=0} \frac{(-1)^n \mathrm{B}_n}{n!} \mathrm{ad}_{\tau^\wedge(t)}^n (\Omega),
\end{equation}
where $\{\mathrm{B}_n\}$ are the Bernoulli numbers. In fact, \cite[Lemma 3]{bullo_murray_1995} shows that in $\text{SO}(3)$ the matrix series \eqref{eq: series_derv} converges to
\begin{align} \label{eq: tau_dot}
    \dot\tau^\wedge &= \mathcal{B}_{\tau^\wedge(t)} (\Omega) \nonumber \\ 
    &= \Omega + \left(\frac{1}{2}\mathrm{ad}_{\tau^\wedge(t)} + \frac{1 - \alpha(\theta)}{\theta^2}\mathrm{ad}^2_{\tau^\wedge(t)} \right) (\Omega),
\end{align}
where $\alpha(\theta) := (\theta/2)\cot(\theta/2)$. This expression will be useful for us thanks to the following identity
\begin{equation}
   \Omega = \tau^\wedge(t) \implies \dot\tau^\wedge = \Omega, 
\end{equation}
since $\mathrm{ad}_{\Omega}(\Omega) = 0$.



\section{Exponential stability and control in $\text{SO}(3)$}




A function $f : \mathbb{R}^+ \rightarrow \mathbb{R}^n$ is said to \textit{converge} to $r_0 \in \mathbb{R}^n$ \textit{exponentially fast} if there exist constants $\alpha, k >0$ and a time $t_0 \in \mathbb{R}^+$, such that for all $t > t_0$
$$\|f(t) - r_0\| \leq \alpha \|f(t_0) - r_0\| e^{-k (t - t_0)}.$$
Hence, considering $f(t)$ as a curve $\gamma(t) \subset M$, we extend this concept of exponential convergence to curves within smooth manifolds as follows. 

\begin{defn}[Local exponential stability in trajectory tracking]
\label{def: exptt}
    Given the signal $\gamma(t) \in M$, we say that the tracking of the desired signal $\gamma_d(t) \in M$ is exponentially stable for the \textit{set of initial conditions} $\mathcal{U}_\epsilon(\gamma_d(t_0)) := \{ x \in M \,|\, d_M(x, \gamma_d(t_0)) \leq \epsilon\}$, where $\epsilon > 0$ and $t_0 \in \mathbb{R}^+$, if there exist constants $\alpha, k >0$ such that
    \begin{equation} \label{eq: exp_stability}
        d_M(\gamma(t), \gamma_d(t)) \leq \alpha d_M(\gamma(t_0), \gamma_d(t_0)) e^{-k(t-t_0)}
    \end{equation}
    for all $\gamma(t_0) \in \mathcal{U}_\epsilon(\gamma_d(t_0))$ and $t \geq t_0$.
\end{defn}

Therefore, from \eqref{eq: exp_stability} we have that $\lim_{t\rightarrow + \infty}\gamma(t) = \gamma_d(t)$; in other words, we can solve Problem \ref{prolem: main} by designing a control law such that, given a local set of initial conditions $\mathcal{U}_\epsilon(R_a(t_0)) \subset \text{SO}(3)$, the tracking of $R_a(t)$ is exponentially stable for all $t>t_0$. For the sake of clarity, we have denoted $R(p(t),t) = R(t)$.

Let us begin by defining an error signal. Given the metric \eqref{eq: metric} and a target attitude matrix
\begin{equation} \label{eq: Ra}
   R_a := \left[x_a \; y_a \; z_a\right] \in \text{SO}(3),
\end{equation}
where $\{x_a, y_a, z_a\}$ is an orthonormal basis of $\mathbb{R}^3$, we define the \textit{attitude error} as 
\begin{equation} \label{eq: mu_r}
   \mu_{R_e} := d_{\text{SO}(3)}(R_a, R) = \|\log(R_e)\|_F,
\end{equation}
where $R_e = R_a^\top  R \in \text{SO}(3)$ is denoted as the \emph{attitude error matrix}. Using $R_e$ in the control input of \eqref{eq: kin} yields the  following feedback system 

\begin{equation} \label{eq: kin_fb}
    \begin{cases}
        \begin{array}{l}
            \dot p = R v \\
            \dot R = R \Omega(R_e) \\
            \dot R_e = R_e \Omega_e
        \end{array},
    \end{cases}
\end{equation}
where $\Omega_e \in \mathfrak{so}(3)$ denotes the velocity tensor of the attitude error matrix, which may depend on the velocity tensor of the desired attitude matrix, defined as $\Omega_a := R_a^\top \dot R_a \in \mathfrak{so}(3)$. Note that by taking the time derivative of $R_e = R_a^\top  R$
\begin{align} \label{eq: R_e_dot}
    \dot R_e &= R_a^{\top}\dot R + \frac{\mathrm{d}}{\mathrm{dt}}(R_a^{\top})R \nonumber\\  
    &= R_a^{\top} \dot R - R_a^{\top} \dot R_a R_a^{\top} R \nonumber\\
    &= R_e(\Omega - R^{\top}\dot R_a R_a^{\top} R),
\end{align}
where we have considered that $\frac{\mathrm{d}}{\mathrm{dt}}(R^{\top}) = \frac{\mathrm{d}}{\mathrm{dt}}(R^{-1}) = R^{-1}\dot R R^{-1} = R^{\top}\dot R R^{\top}$. Thus, since $\dot R_e = R_e \Omega_e$, from \eqref{eq: R_e_dot} we can formulate the angular velocity tensor of $R_e$ as
\begin{align} \label{eq: omega_e}
    \Omega_e &= \Omega - R^{T}\dot R_a R_a^{T} R \nonumber\\
    &= \Omega - R_e^{T} \Omega_a R_e \nonumber\\
    &= \Omega - \mathrm{Ad}_{R_e^{T}}(\Omega_a),
\end{align}
which by taking the vee map while considering \eqref{eq: adjoint_vee} yields
\begin{align} \label{eq: omega_e_vee}
     \Omega_e^\vee = \Omega^\vee - R_e^{T}\Omega_a^\vee,
\end{align}
therefore, we can express the velocity tensor of $R_e$ as a vector in $\mathbb{R}^3$ by taking the difference between $\Omega^\vee$ and $R_e^\top \Omega_a^\vee$.

Once we have introduced our error signal $\mu_{R_e(t)}$ and the feedback system \eqref{eq: R_e_dot}, we are in place to formalize a control law that will allow a robot modeled by \eqref{eq: kin} to converge to a desired time-varying target attitude $R_a(t)$.

\begin{prop}[3D attitude controller] \label{prop: att_control}
    Consider the attitude target $R_a(t)$ and the feedback system \eqref{eq: kin_fb}, for any initial condition $R(t_0)$ such that $\mathrm{tr}(R_e(t_0)) \neq -1 $, the tracking of $R_a(t)$ is exponentially stable under the control law
    \begin{align} \label{eq: omega_control}
        \Omega(R_e) &= - k_w \log(R_a^{T}R) + R^{T} \dot R_a R_a^{T} R \nonumber \\
        &= - k_w \log(R_e) + \mathrm{Ad}_{R_e^\top } (\Omega_a),
    \end{align}
    where $k_w \in \mathbb{R}^+$ is a positive constant that modulates the aggressiveness of the proportional controller.
\end{prop}
\begin{proof}
    Given the Lyapunov function $V(R_e) = \mu_{R_e}^2/2$, we have that
    \begin{equation} \label{eq: V}
        V(R_e) = \frac{1}{2} \|\log(R_e)\|^2 = \frac{1}{2} \mathrm{tr}(\log(R_e)^\top \log(R_e)),
    \end{equation}
    which is well-defined for all $R_e$ if $\mathrm{tr}(R_e) \neq -1$. Taking the time derivative of this expression leads us to
    \begin{align} \label{eq: V_dot_B}
        \dot V(R_e) &= \mathrm{tr}(\log(R_e)^\top \frac{\mathrm{d}}{\mathrm{dt}}\log(R_e)) \nonumber\\
        &= \mathrm{tr}(\log(R_e)^\top \mathcal{B}_{\log(R)}(\Omega_e)) \nonumber \\
        &= \mathrm{tr}(\log(R_e)^\top  \Omega_e),
    \end{align}
    where we have used Lemma \ref{lem: ad_inv} so that
    \begin{align*}
        \langle\langle \mathrm{ad}_{\log(R)}(\Omega) , \log(R) \rangle\rangle &= - \langle\langle \Omega , \mathrm{ad}_{\log(R)}(\log(R)) \rangle\rangle\\
        \langle\langle \mathrm{ad}^2_{\log(R)}(\Omega) , \log(R) \rangle\rangle &= \langle\langle \Omega , \mathrm{ad}^2_{\log(R)}(\log(R)) \rangle\rangle,
    \end{align*}
    and the fact that $\mathrm{ad}_{\log(R)}(\log(R)) = 0$. Therefore, if we let
    \begin{align} \label{eq: omega_e_kw}
        \Omega_e = - k_w \log(R_e),
    \end{align}
    then we have that $\dot V(R_e) = - 2 k_w V(R_e)$. The solution of this ODE is $V(R_e(t)) = V(R_e(t_0)) e^{-2 k_w t}$, which is equivalent to say that $\mu_{R_e}(t) = \frac{1}{\sqrt{2}} \mu_{R_e}(t_0) e^{-k_w t}$. Hence, given the Definition \ref{def: exptt}, we have that the tracking of $R_a(t)$ is exponentially stable if $\mathrm{tr}(R_e(t_0)) \neq -1$ because of Lemma \ref{lemm: tr_log}. 
\end{proof}

In fact, this proposition enables us to present the following formal solution for Problem  \ref{prolem: main}.

\begin{theorem} \label{thm: 1}
    Consider the attitude target $R_a(p(t),t) = R_a(t)$ such that $\dot R_a(t)$ is known. If there exist a time $t_0 \in \mathbb{R}^+$ such that $\mathrm{tr}(R_e(t_0)) \neq -1$, then taking $\Omega(t)$ as in \eqref{eq: omega_control} solves Problem \ref{prolem: main} for every $t \geq t_0$.
\end{theorem}







\section{Practical problems} 

\subsection{Aligning to a vector field with an unknown time variation}

Given a, possibly time-varying, vector field $m_d(t) \in S^2$, whose time variation is unknown but bounded, i.e., $\|\dot m_d(t)\| \leq \omega_d$, where $\omega_d \in \mathbb{R}^+$, our objective now is to align the heading direction of our robot, i.e., $\hat {\dot p} := \dot p/\|\dot p\|$, with $m_d(t)$. Considering Remark \ref{rem: v} and $R(t) = \begin{bmatrix}x & y & z \end{bmatrix}$ with $x,y,z$ as defined in \eqref{eq: R}, we have that $\hat {\dot p}^\top x = 1$, and $\hat {\dot p}^\top y = \hat {\dot p}^\top z = 0$. Thus, this problem is equivalent for us to minimize the geodesic distance \eqref{eq: dis_s2} between $x(t)$ and $m_d(t)$. Indeed, we define the \textit{target attitude set} as follows.

\begin{defn} [Target attitude set]
    Consider $x(t)\in S^2, R_a(t) \in SO(3)$ with $R_a(t) = \begin{bmatrix}m_d(t) & * & *\end{bmatrix}$ and a positive constant $\delta \in [0,\pi]$, the target attitude set is defined as
    \begin{equation*}
        \mathcal{M}_\delta := \{R \,|\, \arccos(x(t)^\top m_d(t)) \leq \delta\}.
    \end{equation*}
\end{defn}
Note that we leave the desired attitude in the YZ axis up to the controller or designer. We now formally introduce the problem of this subsection as follows.



\begin{problem} \label{prolem: omega_unk}
    Consider a robot modeled by \eqref{eq: kin} and a target attitude set $\mathcal{M}_{\delta^*}$ such that $\|\dot m_d(t) \| \leq \omega_d$. Then, design a control signal $\Omega(t)$ such that, for a given $\delta^* \in [0,\pi]$ then $R(t) \to \mathcal{M}_{\delta^*}$ as $t\to\infty$.
\end{problem}




Firstly, let us note that, considering $R_a(t)$ as in \eqref{eq: Ra}, with $x_a(t) = m_d(t)$, we can split the attitude target velocity into

\begin{equation} \label{eq: omega_a_unk}
    \Omega_a = \Omega_a^u + \Omega_a^k,
\end{equation}
where $\Omega_a^u$ is an unknown variation of $R_a(t)$ associated to the time variation of $m_d (t)$, i.e., $\|{\Omega_a^u}^\vee\| \leq \omega_d$, and $\Omega_a^k$ can be configured through the design of the, possibly time-varying but known, vectors $y_a(t)$ and $z_a(t)$. Since we do not know $\Omega_a^u$, we reformulate the control law \eqref{eq: omega_control} as
\begin{equation} \label{eq: omega_control_know}
    \Omega(R_e) = - k_\omega \log(R_e) + \mathrm{Ad}_{R_e^\top} (\Omega_a^k),
\end{equation}
which leads us to the following technical result.




\begin{lemma} \label{lem: exp_stab}
 Given the feedback system \eqref{eq: kin_fb}, an attitude target $R_a(t)$, whose velocity $\Omega_a$ is given by \eqref{eq: omega_a_unk}, and a constant angular velocity $\omega_d \in \mathbb{R}^+$ such that $\|{\Omega_a^u}^\vee\| < \omega_d$. If $\mathrm{tr}(R_e(t_0)) \neq -1$, then the control law in \eqref{eq: omega_control_know} makes $R(t) \to M_{\delta^*}$ exponentially fast as $t\to\infty$ with $k_\omega = \sqrt{2}\frac{\omega_d}{\mu^*_{R_e}}$, where $\mu^*_{R_e} \leq \delta^*$.
\end{lemma}

\begin{proof} Substituting the vee form of \eqref{eq: omega_a_unk} and \eqref{eq: omega_control_know} into \eqref{eq: omega_e_vee}, we obtain
\begin{align} \label{eq: omega_e_unk}
    \Omega_e^\vee &= \Omega^\vee - {R_e^\top}^\vee \Omega_a^\vee \nonumber \\
    &= - k_w \log(R_e)^\vee - {R_e^\top}^\vee {\Omega_a^u}^\vee.
\end{align}
Consider the Lyapunov function $V = \frac{1}{4}\mu_{R_e}^2$ and its time derivative, similarly as in \eqref{eq: V_dot_B} 
\begin{equation*}
    \frac{1}{4}\frac{\mathrm{d}}{\mathrm{dt}} \mu_{R_e}^2 =  \frac{1}{2}\mathrm{tr}\left(\log(R_e)^\top \Omega_e\right) =  \log(R_e)^\vee \cdot \Omega_e^\vee.
\end{equation*} 
Therefore, given $\Omega_e^\wedge$ as in \eqref{eq: omega_e_unk} and considering \eqref{eq: frob_product}, we have that
\begin{equation} \label{eq: to_max}
    \log(R_e)^\vee \cdot \Omega_e^\vee = - \frac{1}{2}k_w \mu_{R_e}^2 - \log(R_e)^\vee \cdot R_e^\top (\Omega_a^u)^\vee,
\end{equation}
where $\frac{1}{2}\mu_{R_e}^2 = \log(R_e)^\vee \cdot \log(R_e)^\vee$. We upper bound the second term of \eqref{eq: to_max} thanks to
$$\log(R_e)^\vee \cdot R_e^\top (\Omega_a^u)^\vee \leq \frac{\mu_{R_e}}{\sqrt{2}}\omega_d,$$
where $\|\log(R_e)^\vee\| = \frac{\mu_{R_e}}{\sqrt{2}}$ because of \eqref{eq: frob_product}; thus, the time derivative of $\mu_{R_e}^2$ can be bounded as
\begin{align*}
    \frac{1}{2}\frac{\mathrm{d}}{\mathrm{dt}} \mu_{R_e}^2 \leq \left(-k_\omega  + \sqrt{2}\frac{\omega_d}{\mu_{R_e}} \right) \mu_{R_e}^2.
\end{align*}
Therefore, if $\mathrm{tr}(R_e(t_0)) \neq -1$, as seen in Lemma \ref{lemm: tr_log}, and knowing that $\mu_{R_e}(t_0)>0$ because it is a distance, then $\mu_{R_e}(t)$ converges exponentially fast as $t\to\infty$ to the interval $[0, \mu_{R_e}^*]$ for $k_\omega = \sqrt{2}\frac{\omega_d}{\mu^*_{R_e}}$, where $\mu^*_{R_e} \leq \delta^*$.  Let us see why this selection of $\mu^*_{R_e}$ is sufficient. First, note that $\mu_{R_e}$ is the length-minimizing geodesic distance, so it is equal to the integrated rotation angle $\theta$, i.e., $\log(R_e)^\vee = \mu_{R_e} l$, where $l\in\mathbb{R}^3$ is the unit vector along the rotation axis. Next, let us decompose the axis $x$ into components parallel and perpendicular to the axis $l$,
\begin{equation*}
    x = x_\parallel + x_\perp = (x^\top l) l - l \times (l \times x).
\end{equation*}
Applying the rotation $\log(R_e)^\vee$ to this vector yields \cite{wikirodrigues}
\begin{align*}
    x_\text{rot} &= x_\parallel + x_{\perp \text{rot}} \\
    &= x_\parallel + \cos(\mu_{R_e}) x_\perp + \sin(\mu_{R_e})l \times v_\perp.
\end{align*}
Therefore, considering $x_\text{rot} = m_d$ we have that
\begin{align*}
    \cos(\delta) = x^\top m_d &= \|x_\parallel\|^2 + \cos(\mu_{R_e})\|x_\perp\|^2\\
    &= \cos(\mu_{R_e}) + \|x_\parallel\|^2(1 - \cos(\mu_{R_e})),
\end{align*}
so $\cos(\delta) \geq \cos(\mu_{R_e})$ and consequently $\delta \leq \mu_{R_e}$. Hence, $\mu^*_{R_e} \leq \delta^*$ ensures that $\delta \leq \delta^*$.


\end{proof}

Now we are ready to provide a solution to Problem \ref{prolem: omega_unk}.

\begin{theorem} \label{thm: 2}
    Given the feedback system \eqref{eq: kin_fb} and attitudes $R_a(t), R(t_0)$ meeting the conditions of Lemma \ref{lem: exp_stab}, then the control law \eqref{eq: omega_control_know} solves Problem \ref{prolem: omega_unk}.
\end{theorem}

\subsection{Distances between robots while they align with a common and fixed attitude}

Even though the following result looks cumbersome because we know the convergence rate to the desired fixed attitude, the result aims to show how to exploit certain norm properties.

\begin{lemma} \label{lem: pij}
    Given two robots $i, j$ with dynamics \eqref{eq: kin}, $||v|| =1$ and a desired attitude matrix $R_a \in SO(3)$, if $\Omega_{\{i,j\}}$ is designed as in \eqref{eq: omega_control} so that the rate of exponential convergence to $R_a$ is $r \in \mathbb{R}^+$, then
    \begin{equation*}
        \|p_{ij}(t) - p_{ij}(t_0)\| \leq \frac{2\sqrt{3}\pi}{r}.
    \end{equation*}
\end{lemma}
    
\begin{proof}
    Given $\dot p_{ij} = (R_i - R_j)v$, the definition of norm for linear applications leads to
    \begin{equation} \label{eq: pij_rij}
        \|\dot p_{ij}\| \leq \|R_i - R_j\|_F\|v\| = \|R_i - R_j\|_F.
    \end{equation}
    If we left multiply $(R_i - R_j)$ by $R_a R_a^\top$, the expression \eqref{eq: pij_rij} yields
    \begin{align*}
        \|R_i - R_j\|_F &\leq \|R_a\|_F\|R_a^\top R_i - R_a^\top R_j\|_F \nonumber\\
        &= \sqrt{3}\|R_e^i - R_e^j\|_F
    \end{align*}
    Hence, since $\|R_i - R_j\|_F$ and \eqref{eq: metric} are bounded equivalent metrics \cite{so3_metrics}, we can consider
    
    \begin{equation*}
        \|R_e^i - R_e^j\|_F \leq \|\log({R_e^j}^\top R_e^i)\|_F,
    \end{equation*}
    which, by the triangle inequality, leads us to

   \begin{align*}
        \|\log({R_e^j}^\top R_e^i)\|_F &\leq \|\log(R_e^i)\|_F + \|\log(R_e^j)\|_F \nonumber\\
        &= \mu_{R^i_e} + \mu_{R^j_e}.
    \end{align*}

    Therefore, since the tracking of $R_a$ is exponentially stable with convergence rate $r$, we have that
    \begin{align*}
        \|\dot p_{ij}(t)\| &\leq \sqrt{3}\left(\mu_{R^i_e}(t_0) + \mu_{R^j_e}(t_0)\right) e^{- r t} \nonumber\\
        &\leq 2\sqrt{3}\pi \, e^{- r t},
    \end{align*}
    so that the relative position will be also stabilized exponentially fast. Indeed, we have that $p_{ij}(t) = p_{ij}(t_0) + \int_0^t  \dot p_{ij}(s) \mathrm{ds}$, so we can conclude that
    \begin{align*}
        \|p_{ij}(t) - p_{ij}(t_0)\| &= \left\|\int_0^t  \dot p_{ij}(s) \mathrm{ds}\right\| \leq \int_0^t  \|\dot p_{ij}(s)\| \mathrm{ds}\\ &\leq 2\sqrt{3}\pi \int_0^\infty  e^{- r s} \mathrm{ds} =  \frac{2\sqrt{3}\pi}{r}.
    \end{align*}
\end{proof}

\subsection{Distance between two robots while they align with a non-totally known vector field}
We finish this tutorial by showing how to compute the distance bound between two robots when they align with a time-varying vector field $m_d(t) \in \mathbb{R}^3$. However, the vector field is only known at time $t$ with $\left\|\dot m_d(t)\right\| \leq \omega_d \in\mathbb{R}^+$, so the robots cannot compensate their tracking with a feed-forward signal.

\begin{prop} \label{prop: local_exp_stability}
    Given two robots $i, j$ with dynamics \eqref{eq: kin}, $||v|| =1$ and a desired attitude matrix $R_a \in SO(3)$. If $\Omega_{\{i,j\}}$ is given by \eqref{eq: omega_control_know}, $R_{i,j}(t_0) \in \mathcal{M}_{\delta^*}$ with $\delta^*$ sufficiently small and $\mu_{R_e^{\{i,j\}}}^*=\mu_{R_e}^*$ is designed as in Lemma \ref{lem: exp_stab}, then 
    \begin{equation*}
        \|p_{ij}(t) - p_{ij}(t_0)\| \leq\frac{2\sqrt{3}\mu_{R_e}^*}{k_w}.
    \end{equation*}
\end{prop}
\begin{proof}
    Given $R_{ij} := R_j^\top R_i$, taking its time derivative yields
    \begin{align*}
        \dot R_{ij} = R_j^\top \dot R_i + \dot R^\top R_i 
        &= R_{ij}\Omega_i - \Omega_j R_{ij} \\
        &= R_{ij} (\Omega_i - Ad_{R_{ij}^\top}(\Omega_j)) =  R_{ij}\Omega_{ij},
    \end{align*}
    and substituting $\Omega_{i}$ and $\Omega_{j}$ by the control law in \eqref{eq: omega_control_know} leads us to
    \begin{align} \label{eq: omega_ij}
        \Omega_{ij} &= -k_w\left(\log(R_e^i) - Ad_{R_{ij}^\top}(\log(R_e^j))\right) \nonumber\\
        &+ Ad_{R_i^\top R_a}(\Omega_a^k) - Ad_{R_{ij}^\top R_j^\top R_a}(\Omega_a^k)\nonumber\\
        &= -k_w\left(\log(R_e^i) - Ad_{R_{ij}^\top}(\log(R_e^j))\right),
    \end{align}
    where $k_\omega^{\{i,j\}}=k_\omega$ because $k_\omega^{\{i,j\}}$ are designed as in Lemma \ref{lem: exp_stab} and $\mu_{R_e^{\{i,j\}}}^*=\mu_{R_e}^*$. On the other hand, let us consider the Lyapunov function
    \begin{align*}
        V(R_e^{ij}) = \frac{1}{2}{\mu_{R_e^{ij}}}^2 = \frac{1}{2}\text{tr}\left(\log(R_e^{ij})^\top \log(R_e^{ij})\right),
    \end{align*}
    where $R_e^{ij} := (R_a^\top R_j)^\top R_a^\top R_i = R^{ij}$. Hence, following the same steps as in \eqref{eq: V_dot_B}, its time derivative leads us to
    \begin{align} \label{eq: V_dot_small}
        \dot V(R_e^{ij}) &= \text{tr}\left(\log(R_e^{ij})^\top \Omega_{ij}\right)\\
        &= -k_w\text{tr}\left(\log(R_e^{ij})^\top \left(\log(R_e^i) - Ad_{R_{ij}^\top}\left(\log(R_e^j)\right)\right)\right)\nonumber.
    \end{align}
    
     Since $\mu_{R_e}^*$ is designed accordingly to Lemma \ref{lem: exp_stab}, we have that if $\delta^*$ is sufficiently small then $\mu_{R_e^{\{i,j\}}}$ are too, i.e., $\mu_{R_e^{\{i,j\}}} \leq \mu_{R_e}^* \leq \delta^*$. Thus, substituting \eqref{eq: Ad_series} in \eqref{eq: omega_ij} yields
    \begin{align*}
        \Omega_{ij} &= -k_w\left(\log(R_e^i) - \sum^\infty_{m=0}\frac{(\mathrm{ad}_{\log(R_{ij})})^m}{m!}(\log(R_e^j))\right)\\
		&= -k_w(\log(R_e^i) - \log(R_e^j)) + \text{H.O.T.},
    \end{align*}
	and considering the \textit{Baker-Campbell-Hausdorff formula} \cite{bch_notes} we have that $\log(R_e^{ij}) = \log(R_e^i) - \log(R_e^j) + \text{H.O.T.}$. Consequently, when $\delta^*$ is sufficiently small \eqref{eq: V_dot_small} yields
    \begin{align*}
		\dot V(R_e^{ij}) &= -k_w \text{tr}\left(\log(R_e^{ij})^\top \log(R_e^{ij})\right) + \text{H.O.T.} \\ &= -k_w V(R_e^{ij}) + \text{H.O.T.},
    \end{align*}
	i.e., $\mu_{R_e^{ij}}(t) = \mu_{R_e^{ij}}(t_0) e^{-k_w t}$. Moreover, when $\delta^*$ is sufficiently small, we know from \cite[Theorem 4.7]{khalil} that the higher order terms are not an issue concerning the exponential convergence. Therefore, following the same steps as in the proof of Lemma \ref{lem: pij} we have that
    \begin{align*}
        \|\dot p_{ij}(t)\| &\leq \sqrt{3} \mu_{R_e^{ij}} (t_0)e^{- k_w t} \\
        &\leq \sqrt{3}\left(\mu_{R^i_e}(t_0) + \mu_{R^j_e}(t_0)\right)e^{- k_w t}\\
        &\leq 2\sqrt{3}\mu_{R_e}^* e^{- k_w t}
    \end{align*}
    thus
    \begin{align*}
        \|p_{ij}(t) - p_{ij}(t_0)\| &\leq 2\sqrt{3}\mu_{R_e}^* \int_0^\infty  e^{- k_w s} \mathrm{ds} \\
        &= \frac{2\sqrt{3}\mu_{R_e}^*}{k_w}.
    \end{align*}
\end{proof}

Ideally, we would like to have (almost) global exponential stability but note that on the proof of Proposition \ref{prop: local_exp_stability} we only show exponential stability when the distance between $R_i$ and $R_j$ is sufficiently small. Therefore, although this result is useful because it shows that $R_i = R_j$ is an expoentially stable equilibrium point, it does not enable us to bound the distance between robots when they start from distant points in $\text{SO}(3)$. Nonetheless, when we focus on 2D unicycle robots, Corollary \ref{cor: 2d} shows that it is straightforward to demonstrate global exponential stability. Consequently, we will show in the simulations of the following section that $\mu_{R_e^{ij}}$ converges to zero exponentially fast.

\begin{coroll} \label{cor: 2d}
    Consider two 2D unicycle robots $i,j$ tracking a vector field $m_d(t)$ with the control law 
    $$\omega = -k_w\delta + \omega_d,$$
    where $\omega$ is the angular heading velocity, $\delta$ is the heading angle error, and $\omega_d = \|\dot m_d(t)\|$. Define $\delta_{ij} := (\delta_j - \delta_i) \in (-2\pi,2\pi)$, so that $\delta_{ij} = 0$ means that robots $i,j$ share the same heading. Hence,
    $$\dot \delta_{ij} = \omega_j - \omega_i = - k_w \delta_{ij},$$
    i.e., $\delta_{ij}(t) = \delta_{ij}(0)e^{-k_w t}$, so $\delta_{ij}$ \textit{globally} converges to zero exponentially fast.
\end{coroll}
For more details on Corollary \ref{cor: 2d} we refer to \cite{acuaviva2023resilient}. In particular, it is shown that the distance 
\begin{equation}
\|p_{ij}(t) - p_{ij}(0)\| \leq \frac{2\pi}{k_\omega}, \, \forall 1\leq i < j \leq N,
\end{equation}
for 2D unicycle robots on the same plane for all initial conditions of $\delta_{ij}$.

\begin{figure}[t]
    \centering
    \includegraphics[trim={0cm 0cm 0cm 0cm}, clip, width=0.98\columnwidth]{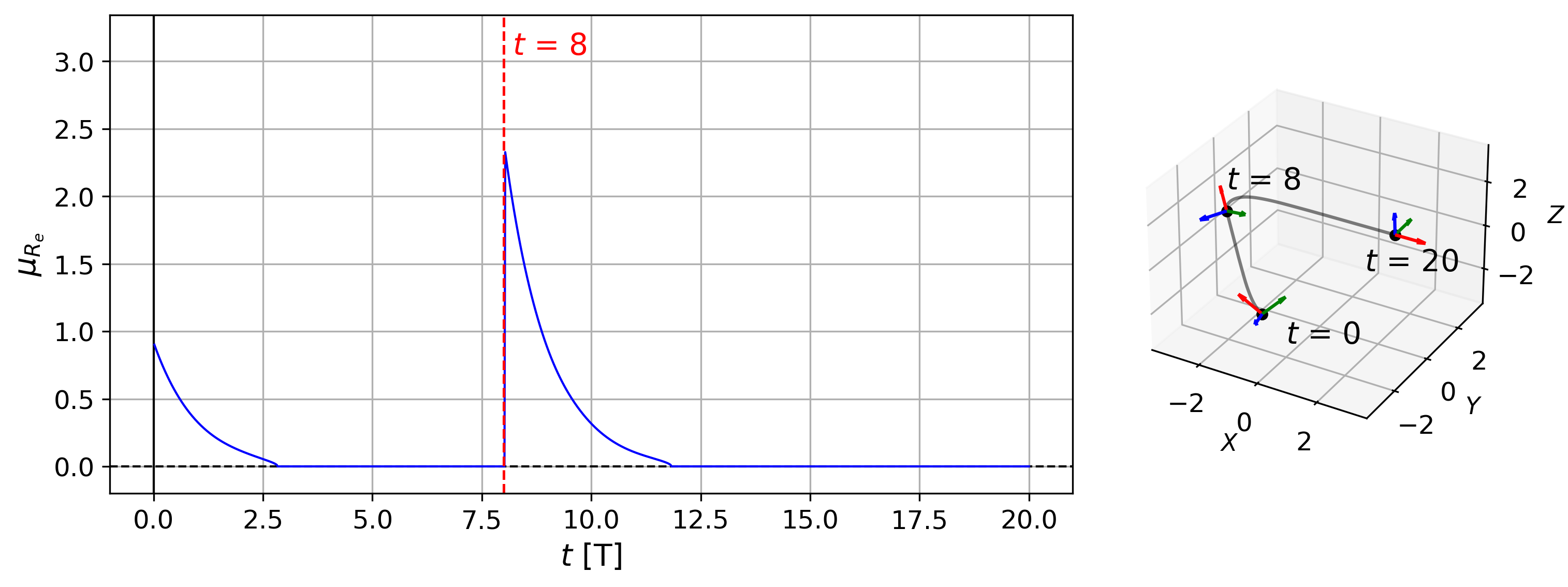}
    \caption{A unicycle robot moving at the constant speed $v = 0.5$ \textit{space unit/time unit} along its $x$ body axis aligns with a time-varying attitude target $R_a(t) = [x_a \; y_a(t) \; z_a(t)]$, with $x_a,y_a,z_a \in S^2$ as in \eqref{eq: R}. The time variation of $y(t)$ and $z(t)$ are given by the body-fixed angular velocity vector $\omega^k = [\pi,0,0]$ in \textit{radians/time unit}, which is known. The control law for alignment is given by \eqref{eq: omega_control}, with $k_\omega = 1$. Note that, at $t=8$, $x_a$ instantly changes from $[-1,1,1]$ to $[1,0,0]$. The left plot shows the time evolution of the attitude error $\mu_{R_e}$. The right plot depicts the trajectory of the unicycle along with its attitude vectors $x$, $y$ and $z$ represented in red, green and blue, respectively.}
    \label{fig: thm1}
\end{figure}

\begin{figure}[t]
    \centering
    \includegraphics[trim={0cm 0cm 0cm 0cm}, clip, width=0.98\columnwidth]{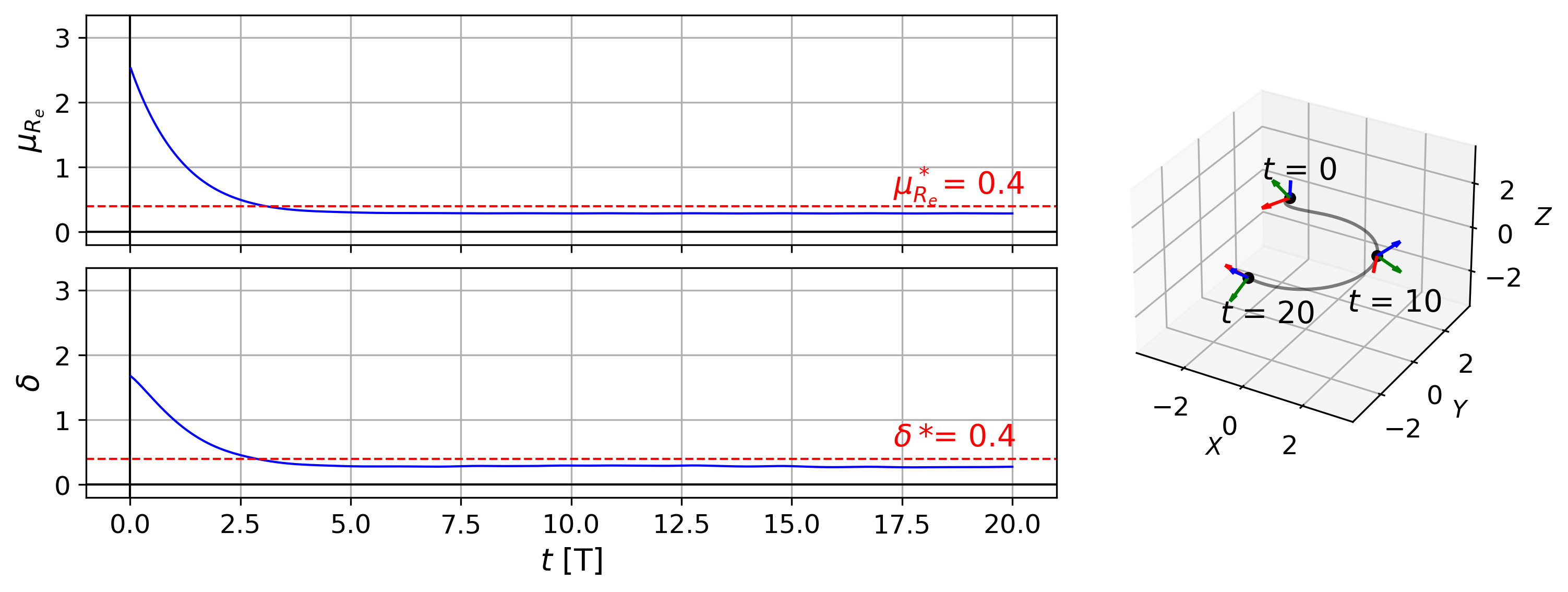}
    \caption{A unicycle robot moving at the constant speed $v = 0.5$ \textit{space unit/time unit} along its $x$ body axis aligns with a time-varying attitude target $R_a(t) = [x_a(t) \; y_a(t) \; z_a(t)]$, with $x_a,y_a,z_a \in S^2$ as in \eqref{eq: R}. The time variation of $R_a(t)$ is given by the earth-fixed angular velocity vector $w(t) = R_a(t)^\top w^k + w^u$, where $\omega^k = [\pi,0,0]$ in \textit{radians/time unit} and $\omega^u = [0,0,-w_d]$ with $w_d = \pi/14$ \textit{radians/time unit}. The known variables are $w^k$ and $\|\omega^u\| = w_d$, so the control law for alignment is given by \eqref{eq: omega_control_know}, with $k_\omega = \sqrt{2}w_d/\mu_{R_e}^*$ and $\mu_{R_e}^* = \delta^* = 0.4$. The left plot shows the time evolution of the attitude error $\mu_{R_e}$ and $\delta$. The right plot depicts the trajectory of the unicycle along with its attitude vectors $x$, $y$ and $z$ represented in red, green and blue, respectively.}
    \label{fig: thm2}
\end{figure}

\begin{figure}[t]
    \centering
    \includegraphics[trim={0cm 0cm 0cm 0cm}, clip, width=0.98\columnwidth]{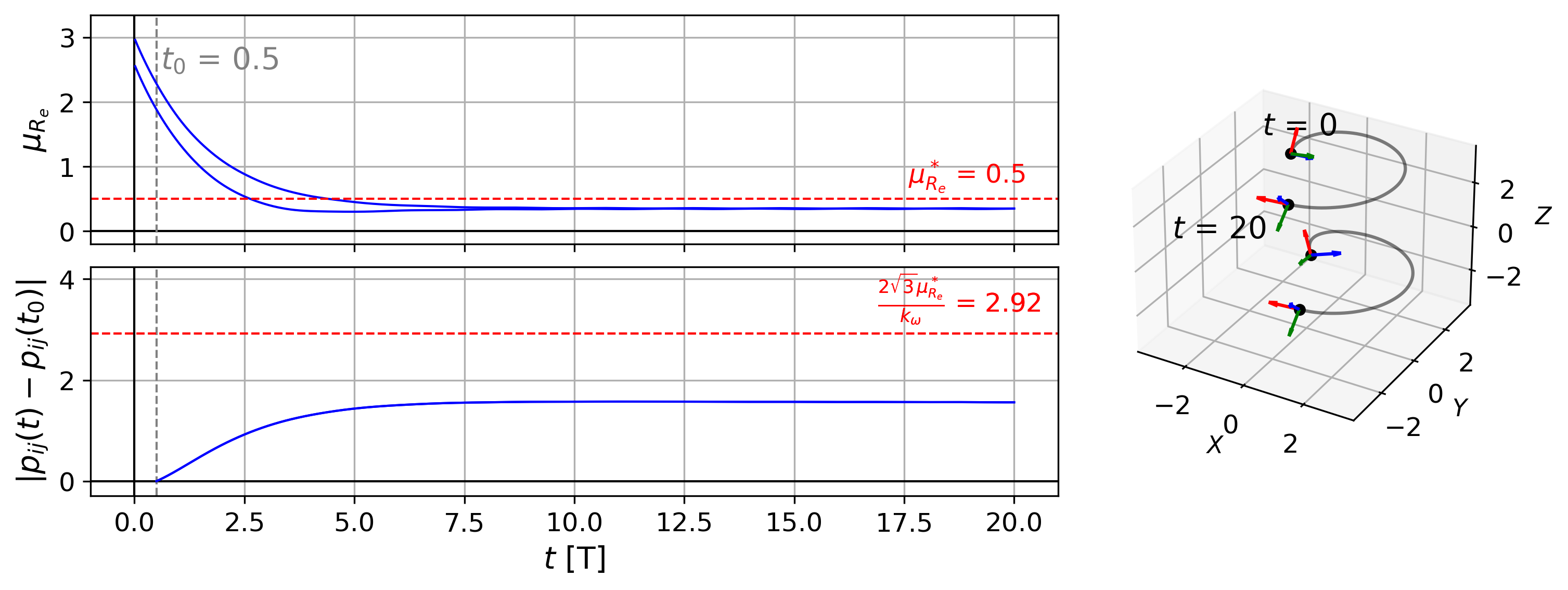}
    \caption{Two unicycle robots moving at the same constant speed $v = 0.5$ \textit{space unit/time unit} along their $x$ body axis align with the time-varying attitude target $R_a(t) = [x_a(t) \; y_a(t) \; z_a(t)]$, with $x_a,y_a,z_a \in S^2$ as in \eqref{eq: R}. The time variation of $R_a(t)$ is given by the earth-fixed angular velocity vector $w(t) = R_a(t)^\top w^k + w^u$, where $\omega^k = [\pi,0,0]$ in \textit{radians/time unit} and $\omega^u = [0,0,-w_d]$ with $w_d = \pi/15$ \textit{radians/time unit}. The known variables are $w^k$ and $\|\omega^u\| = w_d$, so the control law for alignment is given by \eqref{eq: omega_control_know}, with $k_\omega = \sqrt{2}w_d/\mu_{R_e}^*$ and $\mu_{R_e}^* = 0.5$ for both robots. The left plot shows the time evolution of the attitude error $\mu_{R_e}$ and $\|p_{ij}(t) - p_{ij}(t_0)\|$ for $t_0 = 0.5$. The right plot depicts the trajectory of the two unicycle robots along with their attitude vectors $x$, $y$ and $z$ represented in red, green and blue, respectively.}
    \label{fig: prop2}
\end{figure}

\begin{figure}[t]
    \centering
    \includegraphics[trim={0cm 0cm 0cm 0cm}, clip, width=0.98\columnwidth]{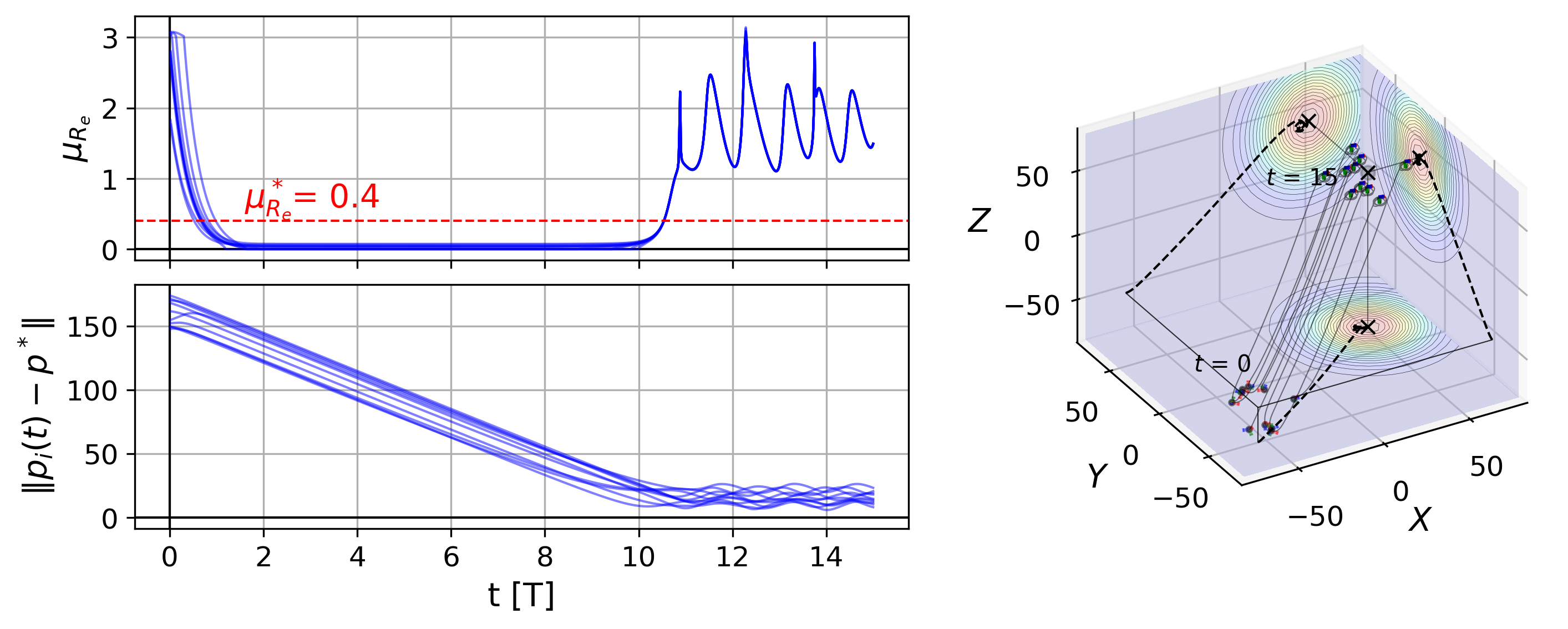}
    \caption{A robot swarm of $N=10$ unicycle robots moving at the same constant speed $v = 15$ \textit{space unit/time unit} along their $x$ body axis seek the source of a scalar field $\sigma(p): \mathbb{R}^3 \rightarrow \mathbb{R}^3$ by aligning with the time-varying attitude target $R_a(t) = [m_d(t) \; y_a(t) \; z_a(t)]$, with $m_d,y_a,z_a \in S^2$ as in \eqref{eq: R}, and where $m_d$ is an ascending direction calculated as in \cite{acuaviva2023resilient}. The time variation of $R_a(t)$ is given by the earth-fixed angular velocity vector $w(t) = R_a(t)^\top w^k + w^u$, where $\omega^k = [\pi,0,0]$ in \textit{radians/time unit} and $w^u \in \mathbb{R}^3$ depends on the scalar field and the position of each agent, as it is deeply explained in \cite{acuaviva2023resilient}. Each robot of the swarm knows $w^k$ and $w_d = \pi/4$ \textit{radians/time unit}, where $w_d \geq \|w^u\|$ when the centroid of the swarm $p_c(t)$ is in a certain set where $\|\nabla\sigma(p_c(t))\| > \epsilon$ for a given $\epsilon > 0$ \cite[Lemma 4.4.]{tfg_antonio}. The control law for alignment is given by \eqref{eq: omega_control_know}, with $k_\omega = \sqrt{2}w_d/\mu_{R_e}^*$ and $\mu_{R_e}^* = 0.4$ for all robots. The left plot shows the time evolution of the attitude error $\mu_{R_e}$ and $\|p_{i}(t) - p^*\|$, where $p^*$ is the source of the scalar field. The right plot depicts the $x$, $y$, and $z$ projections of the scalar field, along with the trajectory of all robots and their attitude vectors $x$, $y$, and $z$ represented in red, green, and blue, respectively.}
    \label{fig: sim_ss}
\end{figure}

\section{Simulations}

In this section, we validate numerically our results by conducting four different simulations. Firstly, in Figure \ref{fig: thm1}, a unicycle robot aligns with an attitude target with a known time variation, which shows that Problem \ref{prolem: main} is solved by Theorem \ref{thm: 1}. In Figure \ref{fig: thm2}, the same unicycle robot aligns with an attitude target with a non-totally known time variation, showing that Problem \ref{prolem: omega_unk} is solved by Theorem \ref{thm: 2}. Another robot is introduced in Figure \ref{fig: prop2} to illustrate that the distance between two unicycle robots is exponentially stable and can be bounded as in Proposition \ref{prop: local_exp_stability}. Finally, in Figure \ref{fig: sim_ss}, we apply all the tools introduced in this work to the source seeking problem in 3D. In the last two simulations, we observe numerically that the exponential convergence of $\mu_{R_e^{ij}}$ to zero is global.


\section{Conclusion}

In this technical note, we introduced roboticists to the essential mathematical tools and techniques for analysing and implementing geometric controllers to align 3D robots with non-constant 3D vector fields. Additionally, we show how to apply this mathematical framework by exploring practical problems where the time variation of the vector field is unknown, and by presenting some technical results with detailed calculations. In future revisions of this work, we aim to extend the proof of Proposition \ref{prop: local_exp_stability} for global exponential stability and to explore different practical problems that will enable us to introduce more advanced analysis tools.

\begin{appendix}

\subsection{The calculation of geodesic distance in $S^2$} \label{ap: s2}

Let $x,y \in S^2$ and an arc $\gamma(t) = (t, h(t)) \subset S^2$, where $h(t): \mathbb{R}^+ \rightarrow \mathbb{R}^2$. Consider that this arc connects $x$ with $y$, so $x = \gamma(t_x) = (t_x, h(t_x))$ and $y = \gamma(t_y) = (t_y, h(t_y))$ for a given $t_y \geq t_x$. Hence, given $\langle x , y \rangle = x \cdot y = x^\top y$ as a valid Riemannian metric for $S^2$, the definition of Riemannian length in \eqref{eq: rieman_len} yields
$$
\ell(\gamma)  = \int \sqrt{\langle \dot\gamma(t) , \dot\gamma(t) \rangle_{\gamma(t)}} \mathrm{d}t = \int \|\dot\gamma(t)\| \mathrm{d}t,
$$
where $\dot \gamma(t) = (1,\dot h(t))$, an so $\|\dot \gamma(t)\| = \sqrt{1 + \|\dot h(t)\|^2}$. 

Note now that $t^2 + h(t)^\top h(t) = \langle \dot\gamma(t) , \dot\gamma(t) \rangle = 1$, so differentiating both sides with respect to $t$ yields
$$
2t + 2h(t)^\top \dot h(t) = 0 \quad \Leftrightarrow \quad t + h(t)^\top \dot h(t) = 0.
$$
Considering this identity in the Cauchy Schwartz inequality leads us to
$$
\|\dot h(t)\|^2 \|h(t)\|^2 \geq \|h(t)^\top \dot h(t)\|^2 = t^2,
$$
so
$$
\|\dot h(t)\|^2 \geq \frac{t^2}{\|h(t)\|^2} \quad \Leftrightarrow \quad 1 + \|\dot h(t)\|^2 \geq \frac{t^2}{\|h(t)\|^2} + 1.
$$
But $\|h(t)\|^2 = 1 - t^2$, thus
$$
\sqrt{1 + \|\dot h(t)\|^2} \geq \sqrt{1 + \frac{t^2}{1 - t^2}} = \sqrt{\frac{1}{1 - t^2}},
$$
and so 
\begin{align*}
    \ell(\gamma) = \int_{t_x}^{t_y} \sqrt{1 + \|\dot h(t)\|^2} \mathrm{d}t 
    &\geq \int_{t_x}^{t_y} \sqrt{\frac{1}{1 - t^2}} \mathrm{d}t
\end{align*}
This last integral represents the minimum length of an arc connecting $x$ to $y$. In fact, this is the definition of length-minimizing geodesic $\gamma_g$. Therefore, if we assume without loss of generality that $x = (a,s,0)$ and $y = (1,0,0)$, then we have 
\begin{align*}
    \ell(\gamma_g) = \int_{a}^{1} \sqrt{\frac{1}{1 - t^2}} \mathrm{d}t = -\arccos(t) \Big|_{a}^{1} &= \arccos(a) \\
    &= \arccos(x^\top y).
\end{align*}

\end{appendix}


\bibliographystyle{IEEEtran}
\bibliography{biblio}

\end{document}